\documentclass[letterpaper, 10 pt, conference]{ieeeconf}  % Comment this line out if you need a4paper

\usepackage{graphicx}
\usepackage{amsmath}
\usepackage{algorithm}% http://ctan.org/pkg/algorithms
\usepackage{algorithmicx}
\usepackage{algpseudocode}
\usepackage{float}
\newtheorem{theorem}{Theorem}

\newtheorem{lemma}[theorem]{Lemma}
\newtheorem{assumption}{Assumption}

\usepackage{cases}

\usepackage{etoolbox}\AtBeginEnvironment{algorithmic}{\small} 
\usepackage{xcolor}
\usepackage{subcaption}
\usepackage{url}

\IEEEoverridecommandlockouts                              % This command is only needed if 
\title{\LARGE \bf
OA-MPC: Occlusion-Aware MPC for Guaranteed Safe\\ Robot Navigation with Unseen Dynamic Obstacles}

\author{Roya Firoozi, Alexandre Mir, Gadiel Sznaier Camps, Mac Schwager
\thanks{This work is submitted on October 17th, 2023. This work was supported in part by ONR grant N00014-18-1-2830. Toyota Research Institute provided funds to support this work.  The first author was supported on an ASEE eFellows fellowship.}
\thanks{Roya Firoozi, is with the Aeronautics \& Astronautics Department at Stanford University, Stanford CA 94305 USA (email: rfiroozi@stanford.edu) }
\thanks{Alexandre Mir, is with the Mechanical Engineering Department, Swiss Federal Institute of Technology Zürich (ETH), 8092 Zurich, Switzerland (email: alexmir@student.ethz.ch) }
\thanks{Gadiel Sznaier Camps, is with the Aeronautics \& Astronautics Department at Stanford University, Stanford CA 94305 USA (email: gsznaier@stanford.edu) }
\thanks{Mac Schwager, is with the Aeronautics \& Astronautics Department at Stanford University, Stanford CA 94305 USA (email: schwager@stanford.edu) }
}

\begin{document}

\maketitle
\thispagestyle{empty}
\pagestyle{empty}

%%%%%%%%%%%%%%%%%%%%%%%%%%%%%%%%%%%%%%%%%%%%%%%%%%%%%%%%%%%%%%%%%%%%%%%%%%%%%%%%
\begin{abstract}
For safe navigation in dynamic uncertain environments, robotic systems rely on the perception and prediction of other agents. Particularly, in occluded areas where cameras and LiDAR give no data, the robot must be able to reason about the potential movements of invisible dynamic agents. This work presents a provably safe motion planning scheme for real-time navigation in an a priori unmapped environment, where occluded dynamic agents are present. Safety guarantees are provided based on reachability analysis. Forward reachable sets associated with potential occluded agents, such as pedestrians, are computed and incorporated into planning. An iterative optimization-based planner is presented that alternates between two optimizations: Nonlinear Model Predictive Control (NMPC) and collision avoidance. The recursive feasibility of the MPC is guaranteed by introducing a terminal stopping constraint. The effectiveness of the proposed algorithm is demonstrated through simulation studies and hardware experiments with a TurtleBot robot equipped with a LiDAR system. The video of experimental results is also available at \tt{\url{https://youtu.be/OUnkB5Feyuk}.}

\end{abstract}

\section{INTRODUCTION}
\label{SEC1}

Robotic systems rely on various types of sensors including range-based sensors such as LiDAR or depth cameras to perceive the surrounding environment. However, sensors are susceptible to occlusion. The field of view of the robot's sensors limits its visibility, and dynamic agents can hide undetected in occluded regions.
% Occlusions can be caused by static objects in the scene such as walls, buildings, trees, etc, or by dynamic obstacles such as pedestrians, cars, or other agents.
In a dynamic environment, safe trajectory planning for a robot requires predicting the future behavior of other agents present in the scene. However, occluded areas may hide undetected dynamic agents, whose trajectory a robot cannot predict. Therefore, to safely navigate in an a priori unmapped environment among dynamic agents where occlusions are present, a robot must reason about all potential future motions of all potential dynamic agents that might be hidden in the occluded regions.  
% In dynamic environment, the occlusion may be produced by other moving agents present in the scene.

\begin{figure}
\centering
\begin{minipage}{\columnwidth}
    \includegraphics[width=1\columnwidth]{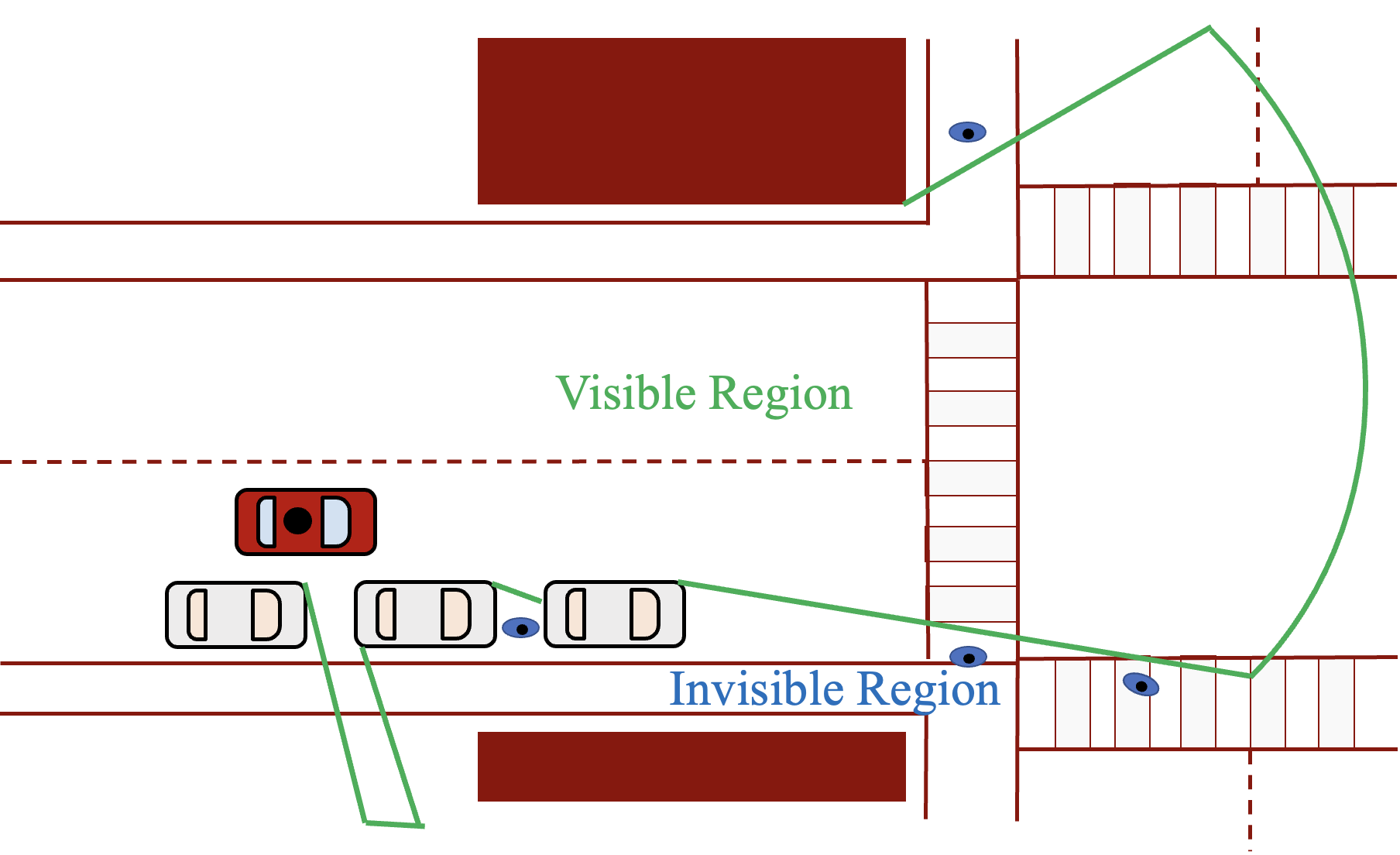}\\
     \vspace{2pt}
    % \centering(a)
\includegraphics[width=1\columnwidth]{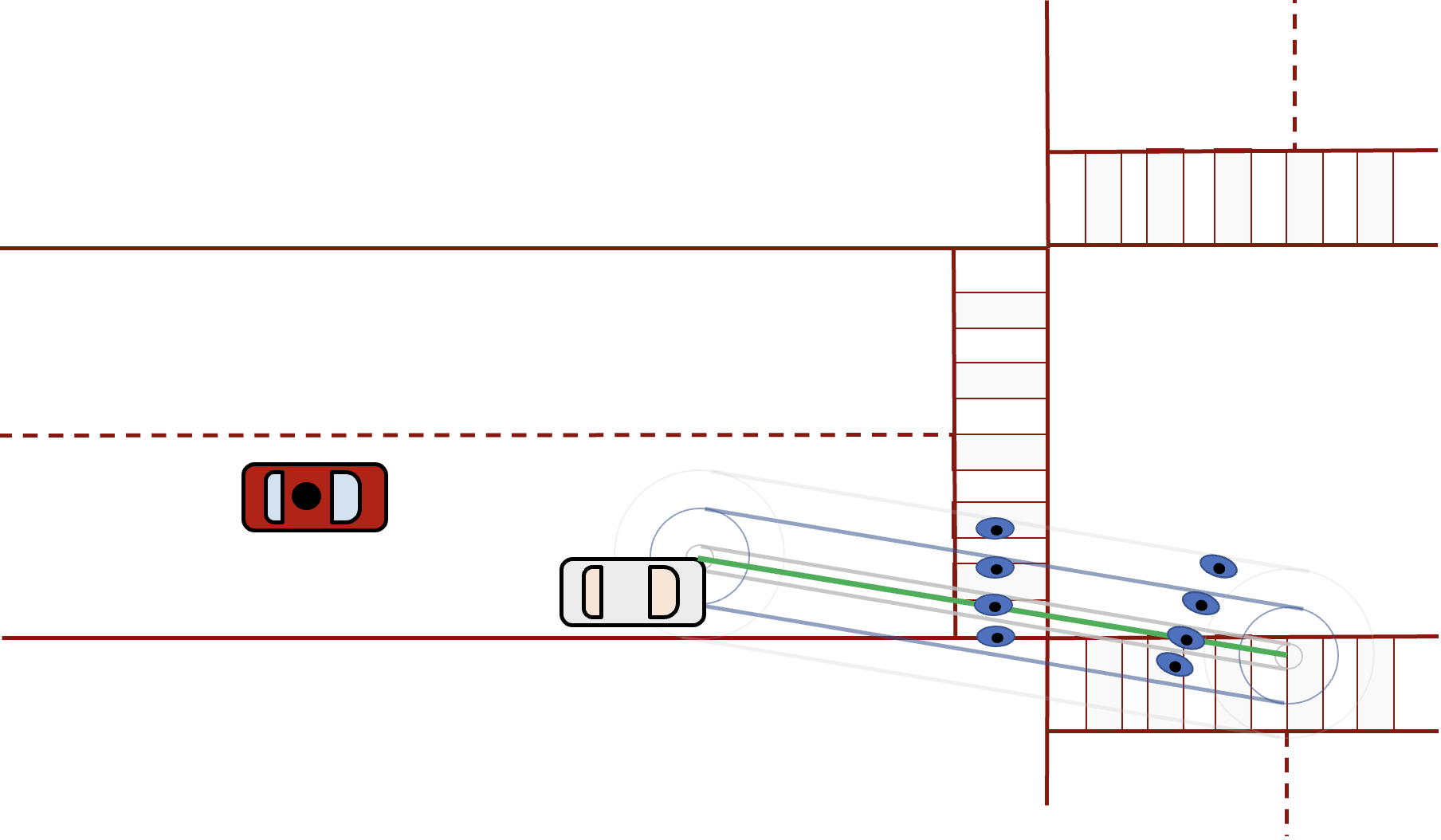}\\
     \vspace{2pt}
    % \centering(b)
\includegraphics[width=1\columnwidth]{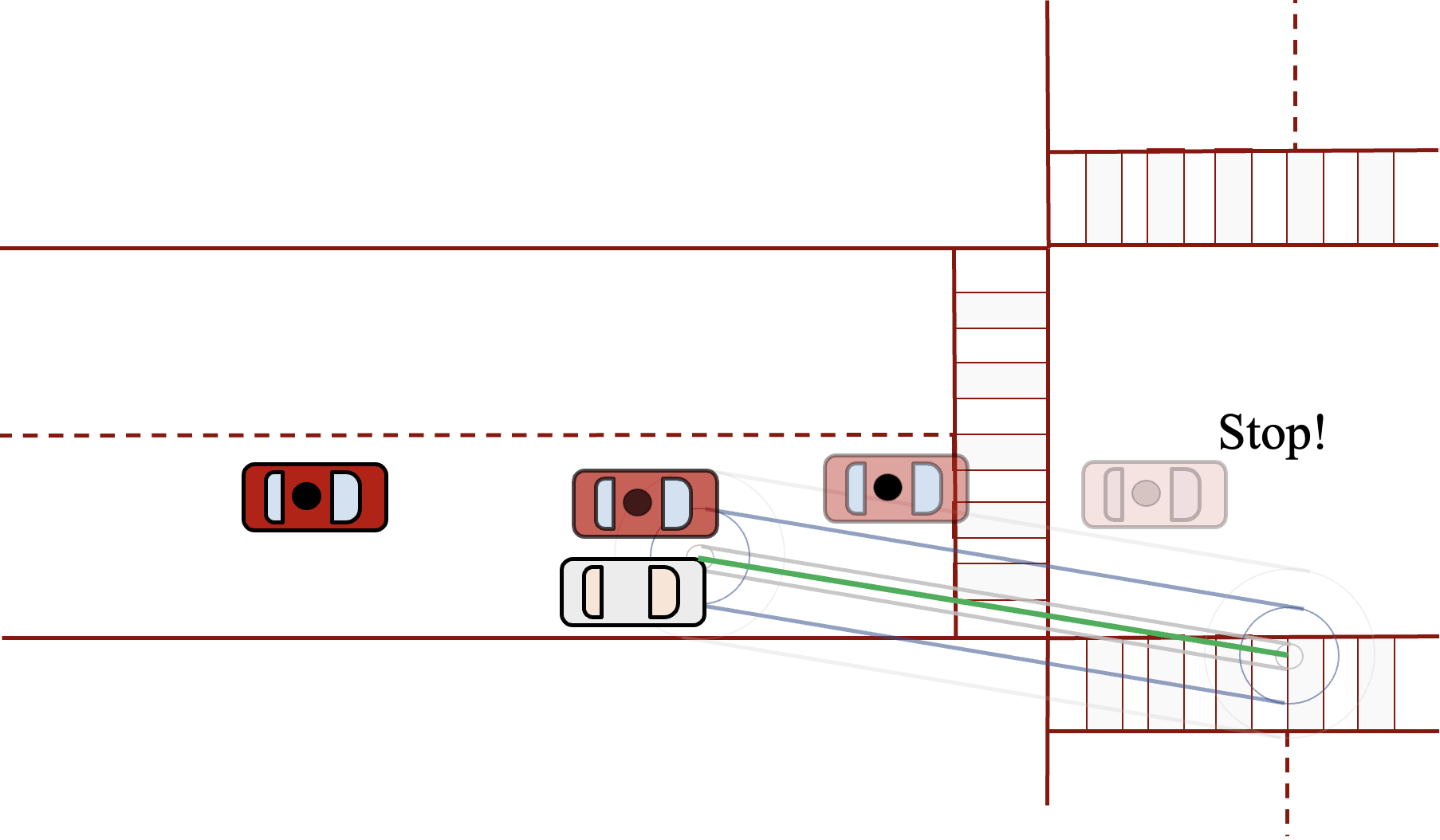}
     \vspace{2pt}
    % \centering(c)
\end{minipage}
\caption{The ego vehicle (red) equipped with a range-based sensor scans the visible regions. \textbf{Top:} Pedestrians (small blue circles) who are hidden in an occlusion are not detected by the sensor. \textbf{Middle:} Nested capsules representing the forward reachable sets for all pedestrians potentially hidden in the occlusion are computed. \textbf{Bottom:} The vehicle optimizes its trajectory while avoiding the nested capsules along its prediction horizon, coming to a stop at the end of the horizon.}
\label{fig:occluded_area}
\end{figure}

This work proposes a framework for general robotics settings including ground robots, autonomous cars, or aerial robots to safely travel through the occluded environment in real-time while avoiding both static and dynamic obstacles. Our approach relies on real-time sensor observations to detect occluded region boundaries from which a potential moving agent might emerge at any time in the future. Forward reachability analysis is performed for dynamic obstacles in both visible and occluded areas. The robot optimizes a trajectory in an MPC fashion which attempts to reach its goal position as fast as possible while it avoids collision with the forward reachable sets that are computed based on the last updated LiDAR scan of the environment. Figure \ref{fig:occluded_area} illustrates an example of our approach in an autonomous driving scenario. The top figure shows that the ego vehicle (red) uses its range sensor (LiDAR or depth camera) to scan the visible regions (the visible region boundary is shown as a solid green line). Dynamic agents or pedestrians (shown as small blue circles) who are outside the visible region are not detected by sensors. To reason about the movement of potential pedestrians in the occluded region, the robot assumes their worst-case behavior, which is moving with maximum speed in any direction from the occlusion boundary. The middle figure shows capsules which are the pre-computed reachable sets resulting from this worst-case assumption on the undetected agents in the occlusion. These capsule sets are constructed on top of the boundary of occlusion and enlarged along the MPC horizon. The nested capsule sets and repeated blue circles represent the reachable sets and potential pedestrian movement along the prediction horizon of MPC. In the bottom figure, MPC prediction steps (open-loop plan) are illustrated, where the vehicle avoids colliding with the growing reachable capsules over the prediction horizon. The minimum distance between the ego car at each step of the MPC horizon is enforced with respect to each associated reachable set along the horizon. Also, we find that requiring the robot to come to a stop at the end of its prediction horizon ensures recursive feasibility for the MPC, although, due to the MPC re-solving at each time step, the vehicle does not actually come to a stop during normal execution. Planning to stop can be seen as a guaranteed safe contingency plan in the unlikely event that the worst-case prediction comes true.

The contributions of the proposed approach are summarized below:
\begin{itemize}
    \item {A novel reachability analysis is proposed based on the line-of-sight of LiDAR at the boundary of the occluded region to determine the worst-case location of all potential occluded agents over a time horizon.}
    \item {A computationally efficient iterative optimization scheme is proposed for online occlusion-aware planning.}
    \item A formal guarantee for recursive feasibility (which ensures collision avoidance) of the proposed MPC scheme is provided. 
\end{itemize}

The paper is organized as follows: Sec.~\ref{SEC2} reviews the related work. Sec.~\ref{SEC3} provides the required notations. Sec.~\ref{SEC4} describes the problem definition. Sec.~\ref{SEC5} provides recursive feasibility guarantees for the proposed MPC scheme. Sec.~\ref{SEC6} presents our proposed iterative optimization algorithm. Sec.~\ref{SEC7} presents the applications and numerical results based on simulation studies and real-world experiments. Sec.~\ref{SEC8} makes concluding remarks and discusses the future work.

\section{Related work}
\label{SEC2}
{\textbf{Reachability Analysis: }Reachability analysis is a powerful set of theoretical tools to analyze safety-critical systems. They have been extensively studied and widely used to achieve safety guarantees \cite{Gillula2010, Fisac2015}.
To perform reachability analysis set-based approaches or Hamilton-Jabobian (HJ) reachability are employed. While set-based approaches are limited to linear dynamics and therefore cannot provide exact reachability for nonlinear systems, Hamilton-Jabobian partial differential equations (PDEs) can handle nonlinear dynamics. However, HJ reachability scales poorly with state dimension and is computationally expensive. Thus, there is a trade-off between providing safety guarantees and computational efficiency and stronger guarantees can be provided with higher computational cost. 

% Hamilton-Jabobian (HJ) reachability is a standard method to analyze safety, but it is computationally expensive.
In particular, for real-time applications such as robot navigation, to address this challenge and reduce the online computation burden, the authors in 
% \cite{Dabadie2014, Ding2010}, and
\cite{Ding2011} propose precomputation of HJ reachable sets offline and storing them in a look-up table to be used online. In \cite{Herbert2017}, the authors present a fast and safe tracking approach (FaSTrack) in which slow offline computation provides safety guarantees for a faster online planner.
FaSTrack includes a complex offline tracking system and a simplified online planning system. A pursuit-evasion game between the two systems is precomputed using HJ reachability and a safety bubble is built around the planning system. Online, based on the relative state of the two systems a safe controller is selected, using the precomputed look-up table. 
In \cite{David2018}, the authors propose to extend the FaSTrack approach to a meta-planning scheme in which a refined offline computation enables safe switching between different online planners. However, these approaches are restricted to static environments and are not applicable to dynamic environments with moving obstacles. In \cite{Zhan-RSS-21}, the authors present a pursuer-evader game theoretic framework for occlusion-aware driving (in a dynamic environment) with safety guarantees using reachability, and simulation studies are provided by assuming linear dynamics (double integrator) for the vehicle model. Control Barrier Functions (CBFs) \cite{ames2017control} are a widely used tool in control synthesis for ensuring safety in critical systems. CBFs convert the safety constraints of control-affine systems into state-feedback constraints that are linear with respect to the control inputs. While this method is widely used as a safety certificate, to the best of our knowledge it is not incorporated with occlusion-aware safety guarantees. Also, the reference governer \cite{GARONE2017306} is another add-on control scheme that provides state and input constraint satisfaction. However, the existing literature does not include occlusion-aware or perception-aware variants of this method as well. Compared to \cite{Ding2011} and \cite{Fisac2015}, we provide a simple reachability analysis as a safety certificate, which is computationally efficient for real-time applications. Our proposed approach is categorized as a set-based reachability method and compared to \cite{Herbert2017} and \cite{David2018}, our method is perception-aware, takes into account sensor limitations, does not require offline computation of reachable sets, and is not limited to static environments and can be used in dynamic environments as well. Compared to \cite{Zhan-RSS-21}, our approach is not limited to linear dynamic models and accommodates nonlinear dynamics as well. 

{\textbf{Active Perception and Planning: }}Joint active perception and planning help the robot to gather additional information about the occluded regions and gain visibility in a partially occluded environment. The authors in \cite{Falanga2018}, propose a perception-aware MPC framework for quadrotors that simultaneously optimizes planning and perception objectives. In \cite{Lee2020_AggressivePerceptionNavigation}, an MPC framework is coupled with a visual pipeline. A deep optical flow dynamics presented as a combination of optical flow and robot dynamics is used to predict the movement of points of interest. The proposed Pixel MPC algorithm controls the robot to accomplish a high-speed task while maintaining the visibility of the important features. In \cite{Higgins2021} a geometric heuristic is proposed to define a visibility measure around a corner. The perception objective (visibility) is added to the MPC multi-objective cost to navigate the occluded corner while maximizing visibility. In our proposed approach visibility maximization is not incorporated explicitly. However, by including a minimum safety distance to the avoid set, the robot indirectly maximizes its visibility to minimize the occluded area.

Other works including \cite{Hoerman2017_CrossroadsBlindCorners} and \cite{Lee2017_CollisionRiskAssessment,Chung2009_SafeNavigation,Brechtel2014_ProbaDecisionMaking,Miller2008_Skynet,Sadou2004_OcclusionsObstacleDetection,Bouraine2014_PassivelySafePlanning}
% ,Zhan2016_NonConservativelyDefenseStrat}
consider the risk caused by occlusions, ranging from simple visibility modeling for pre-detected obstacles like \cite{Miller2008_Skynet}, to more complex multi-layered models proposed in \cite{Hoerman2017_CrossroadsBlindCorners}. 
These works minimize the risk of collision but they do not provide formal safety guarantees. In \cite{Andersen2017}, MPC planning is proposed in which the objective is to maximize visibility,
% or field-of-view, while safety constraints are checked by a higher-level state machine
but formal safety guarantees or recursive feasibility of MPC is not provided. Compared to these online navigation approaches, our proposed MPC provides safety guarantees by online computation of reachable sets. In \cite{Nager2019}, an inference and motion planning scheme is proposed to guarantee safety with respect to hypothetical hidden agents.
% that have not been observed using sensor scans.
Reachable sets are constructed by taking an occupancy map and enlarging the occluded boundaries.
However, this work assumes the availability of a class of possible intentions for other agents which is used in the prediction step of inference to evolve the dynamics of other agents. 
Compared to this work, our approach does not make any assumptions regarding the intention of hidden agents. Our proposed framework is dynamic-agnostic with respect to hidden agents and assumes the hidden agents may move in any arbitrary direction. 
% a novel MPC-based framework to navigate occluding environments that increases visibility while considering uncertainties. Uncertainties are considered through an occupancy mapping-based approach to autonomously decide if it is safe to move around an occluding corner. If it is not safe, the policy framework chooses a motion that promotes visibility as it approaches the occlusion, thereby reducing uncertainty while navigating around a corner safer and faster.
In  \cite{Penin2020_AggressiveTargetTracking} the objective is to follow a moving target while avoiding obstacles in the environment and occlusions in the image space, in a smooth but reactive way. A vision-based approach is implemented on a real quadrotor on multi-objective optimization, with an occlusion constraint formulation, keeping a target in its camera's field of view. Inspired by the coordination of human eyes, the work \cite{Chang2020_CrossDroneCoordination} has explored the visibility in multi-vehicle frameworks. Their idea is to use distributed target measurements to overcome the occlusion effects. The approach of recognizing and supervising a geometrically complex environment by flying UAVs next to it has also been described in the \cite{Zhang2019_OcclusionAwareUAV}'s work. The work \cite{Fehr2014_OcclusionAlleviation} even equipped a robot with a sensor capable of searching for better points of view to facilitate object classification in occluded environments.

In this context of the unknown uncertainty, \cite{Janson2018_OptimalityBenchmarks} proposes a planning trajectory when the obstacle state is a-priori unknown. This work considers the occlusions by static obstacles, and not dynamic ones, that could appear from behind the vehicle for example.
Considering the behavior of other vehicles and of humans in autonomous transportation, \cite{Orzechowski2018_SetBasedSafety} calculated the reachability sets from an occluded area in order to avoid collisions. This work takes into consideration that all cars obey traffic laws and therefore cannot be generalized to all maps.

\textbf{Computation Efficiency and Iterative Optimization: }Another challenge in designing algorithms for robot navigation is ensuring real-time performance. In particular, using the MPC method in control design requires solving the MPC optimization problem in real time. Imposing collision avoidance constraints or considering a nonlinear dynamic model \cite{Firoozi2018} can lead to a nonlinear optimization formulation, which can be computationally prohibitive in some cases \cite{FIROOZI2021104714}. In such cases, additional reformulation is required to develop a framework suitable for real-time applications. Alternation methods in optimization such as the Alternating Direction Method of Multipliers \cite{Boyd2011ADMM} are techniques where optimization problems are solved by alternating between different variables or subsets of variables iteratively. These methods can reduce the computation time by iteratively solving smaller optimization problems. Particularly for robotics applications and the use of MPC, \cite{Firoozi2024} shows alternation between the NMPC optimization and collision avoidance optimization in an iterative manner can significantly improve the computation time and ensure real-time performance. Similar to \cite{Firoozi2024}, we have shown real-time performance of our proposed NMPC framework using iterative optimization.
\section{Notation}
\label{SEC3}
The required definitions and notations are provided below.\\
\emph{MPC Scheme:} MPC is useful for online motion planning among dynamic obstacles because it is able to re-plan according to the newly available sensory data. MPC relies on the receding-horizon principle. At each time step it solves a constrained optimization problem and obtains a sequence of optimal control inputs that minimize a desired cost function $l$, while considering dynamic, state, and input constraints, over a fixed time horizon. Then, the controller applies, in a closed loop, the first control-input solution. At the next time step, the procedure is repeated. In this paper, $\mathbf{u}_{0:N-1}$ and $\mathbf{z}_{0:N}$ indicate the input and state values along the entire planning horizon $N$, predicted based on the measurements at time $t$. For example $[\mathbf{z}_{0|t}$ $\mathbf{z}_{1|t}$,$...,\mathbf{z}_{N|t}]$ represents the entire state trajectory along the horizon $N$ predicted at time $t$ (referred to as the open-loop trajectory). Throughout this paper $\tau$ represents the absolute time, $t$ is the MPC execution time and $k$ is the open-loop prediction step for a given time $t$. The static obstacles are denoted as $\mathcal{O}$ and dynamic agents' reachable sets are represented as $\Omega$. So the unsafe set is defined as the union of these two sets $\Omega \cup \mathcal{O}$. The safe set $\mathcal{Z}^s$ is represented as the complement of the unsafe set $\mathcal{Z}^s = (\Omega \cup \mathcal{O})^c$. Superscript $c$ denotes the complement of the set.\\
% The MPC terminal set is denoted as $\mathcal{Z}_f$.\\
\emph{Occluded Area:} As the robot navigates in the environment, the occluded regions are identified using a LiDAR or depth camera sensor. The boundary of the occluded area can be used to perform reachability analysis for hidden occluded agents. Similarly forward reachability analysis can be performed for visible agents, in case no prediction knowledge is assumed regarding their future actions. \\ 
\emph{Forward Reachable Sets:} The potential locations of dynamic agents (visible or invisible) in the environment are captured using forward reachability analysis. In general, a forward reachable set $\Omega_N$ is defined as the set of states that are reachable after $N$ time steps by a dynamical system $z_{k+1} = f(z_k,u_k)$ from a given set of initial states $\Omega_0$ by applying an admissible sequence control input $u_k \in \mathcal{U}$, specifically 
% \begin{align}
$\Omega_N = \{z_N : \exists u_k \in \mathcal{U} \,\, \forall k = 0, \ldots, N-1,
% \nonumber{}
\quad \text{s.t.} \,\, z_{k+1} = f(z_k,u_k),
z_0 \in \Omega_0\}$.
% \nonumber$
% \end{align}
It follows from this definition that\footnote{Proof: Consider $z_k \in \Omega'_k$, a sequence of inputs $(u_k, \ldots, u_{N-1})$, and the resulting reachable point $z_N \in \Omega'_N$.  Since $\Omega'_k \subset \Omega_k$, we know $z_k \in \Omega_k$.  Since $z_N$ is reachable through the same set of control inputs, $z_N \in \Omega_N$.  This is true for any $z_N \in \Omega'_N$, so $\Omega'_N \subset \Omega_N$.} 
\begin{equation}
    \label{Eq:NestedReachableSets}
 \Omega'_k\subset \Omega_k \implies \Omega'_N \subset \Omega_N \,\, \text{for all} \,\, N \ge k.
\end{equation}
% \emph{Occluded Area:} As the robot navigates in the environment, the occluded regions are identified using LiDAR sensor. In Figure \ref{fig:capsules_sets_OB} the robot (magenta circle) is navigating in an environment where both hidden and visible agents exist. The obstacles (blue regions) cause occlusion. The boundary of occluded area (solid red lines) can be used to perform reachability analysis for hidden occluded agents. In Similar forward reachability analysis can be performed for visible agents, in case that no prediction knowledge is assumed regarding their future actions. The forward reachable sets for visible agent are shown as gray circles. 
% \begin{figure}[ht]
% \centering
% \includegraphics[width=0.3\textwidth]{Pics/caps_visible_ped.png}
% \caption{Occlusion boundaries are shown as solid red lines. Reachable sets for occluded agents are shown in shades of red and reachable sets for visible dynamic agent are shown in gray.}
% \label{fig:capsules_sets_OB}
% \end{figure}

\section{Problem Formulation}
\label{SEC4}
Safe navigation in a dynamic occluded environment can be formulated as an NMPC optimization problem that computes collision-free trajectories in real time. The MPC optimization problem is formulated as follows 
\begin{subequations}\label{eq:MPC_formulation}
\begin{align}
& \min_{\substack{\mathbf{u}_{t:t+N-1}}}
% & \min_{\substack{(u_{0|t},...,u_{N-1|t})}}
\label{eq:total_cost}
& &\sum _{k = 0}^{N}l(\mathbf{z}_{k|t},\mathbf{u}_{k|t}) \\
\label{eq:dynamic_constraint}
&\textrm{subject to} & & \mathbf{z}_{k+1|t} = f(\mathbf{z}_{k|t},\mathbf{u}_{k|t}),\\
\label{eq:initial_cond}
& & & \mathbf{z}_{0|t} = \mathbf{z}(t),\\
\label{eq:state_input_bound}
& & & \mathbf{z}_{k|t} \in \mathcal{Z}, \ \mathbf{u}_{k|t} \in \mathcal{U},\\
% \label{CA_static_env}
% & & & \mathcal{C}(\mathbf{z}_{k|t})\cap \mathcal{O}^b_{k|t} = \emptyset, \ b \in \mathcal{O},\\
% \label{obstacle_avoidance_constraint}
% & & & \mathcal{C}(\mathbf{z}_{k|t})\cap \Omega^r_{k|t} = \emptyset, \ r \in \mathcal{S},\\
\label{eq:complementarity_constraint}
& & & \|\mathbf{z}_{k|t}-\mathbf{z}_{k-1|t}\|g_t(\mathbf{z}_{k|t}) \le 0,\\
\label{eq:final_constraint}
& & & {\mathbf{z}}_{N|t} = \mathbf{z}_{N-1|t},\\
& & & \notag \nonumber{\textrm{and~}  k \in \{0,1,2,..,N\}},
\end{align} \end{subequations}
where $\mathbf{u}_{t:t+N-1} = [u_{0|t},...,u_{N-1|t}]$ denotes the sequence of control inputs over the MPC planning horizon $N$. The robot state and input variables $\mathbf{z}_{k|t}$ and $\mathbf{u}_{k|t}$ at step $k$ are predicted at time $t$. The function $f(\cdot)$ in \eqref{eq:dynamic_constraint} represents the nonlinear (dynamic or kinematic) model of the robot, which is discretized using Euler discretization. $\mathcal{Z}$, $\mathcal{U}$ are the state and input feasible sets, respectively. These sets represent state and actuator limitations. Constraint \eqref{eq:complementarity_constraint} is the complementarity constraint in which $g_t(\mathbf{z}_{k|t}) \le 0$ represents all the collision avoidance constraints. The time-varying safe set $\mathcal{Z}^s_{k|t}$ is defined as $\mathcal{Z}^s_{k|t} = \{ {g_t(\mathbf{z}_{k|t}) \leq 0, \ \forall \mathbf{z}_{k|t}}\}$. The complementarity constraint \eqref{eq:complementarity_constraint} formulation allows for $g_t(\mathbf{z}_{k|t})$ to be greater than $0$ (the collision avoidance constraint is violated) as long as $\mathbf{z}_{k|t} = \mathbf{z}_{k-1|t}$ (the robot is stopped). As an example consider a scenario where a pedestrian hits a stopped car, the collision avoidance constraint is violated but since the car is stopped, we consider the plan safe.  The idea is that if the pedestrian hits a stopped car, likely there will not be a major injury, and the pedestrian bears the fault in the collision, not the car. Constraint \eqref{eq:final_constraint} is the terminal constraint that is required to guarantee recursive feasibility. Details are discussed in the next section. The collision avoidance constraint $g_t(\mathbf{z}_{k|t}) \le 0$ includes static and dynamic constraint as follows: 
\begin{subequations}
\begin{align}
\label{CA_static_env}
& \mathcal{C}(\mathbf{z}_{k|t})\cap \mathcal{O}^b = \emptyset, \ b \in \mathcal{B},\\
\label{obstacle_avoidance_constraint}
& \mathcal{C}(\mathbf{z}_{k|t})\cap \Omega^r_{k|t} = \emptyset, \ r \in \mathcal{S}.
\end{align}
\end{subequations}
Constraints \eqref{CA_static_env} are collision-avoidance constraints between the robot $\mathcal{C}$ (modeled as a circle) and the static environment denoted by set $\mathcal{B}$ indexed by $b$.  Constraints~\eqref{obstacle_avoidance_constraint} represent the collision-avoidance constraints between the robot $\mathcal{C}$ and hidden and visible dynamic obstacles in the scene denoted by the set $\mathcal{S}$ indexed by $r$. 
% This representation is time-varying and is a function of the robot state at each time step. 
% The complementarity constraint \eqref{eq:complementarity_constraint} defined as $\|\mathbf{z}_{k|t}-\mathbf{z}_{k-1|t}\|g(\mathbf{z}_{k|t}) \le 0,$ where $g(\mathbf{z}_{k|t}) \le 0$ represents all the collision avoidance constraints.
% % the terminal state constraint as $\mathbf{z}_{N|t} = \mathbf{z}_{N-1|t}$. 
% (This formulation allows for $g(\mathbf{z}_{k|t})$ to be greater than $0$ as long as $\mathbf{z}_{k|t} = \mathbf{z}_{k-1|t}$ (the robot is stopped)). \\
% Note that the NMPC problem \eqref{eq:MPC_formulation}, might get infeasible. However, persistent feasibility of \eqref{eq:MPC_formulation} can be guaranteed by computing the reachable set. Guarantees of recursive feasibility of MPC is discussed in the next section.
% the terminal state constraint as $\mathbf{z}_{N|t} = \mathbf{z}_{N-1|t}$.
\section{Recursive Feasibility Guarantees}
\label{SEC5}
\begin{assumption}
\label{As:Sensor}
We assume the robot's sensor at each time $t$ returns a set $\bar{\Omega}^r_{0|t}$ that contains agent $r$ at time $t$.  If $r$ is not occluded, this set describes the region implied by the sensor measurement (e.g., a disk with some error radius).  If $r$ is occluded, this set is the occlusion region itself. If a sensor measurement is not taken at time $t$, $\bar{\Omega}^r_{0|t}$ is the entire admissible space for the agent.\\
\end{assumption}
Before we prove the main results of recursive feasibility, we state and prove a lemma concerning the evolution of the safe set for the robot. 
\begin{lemma}\label{lemma1}
The predicted safe set at absolute time $\tau$ is non-decreasing as $t$ increases, that is, 
\begin{equation}
\mathcal{Z}^s_{(\tau-t_1)|t_1}\subset \mathcal{Z}^s_{(\tau-t_2)|t_2} \quad \forall t_1 \le t_2 \le \tau.
\end{equation}

\begin{proof}We only have to show that the forward reachable set of dynamic agents is non-increasing,
\begin{equation}
    \Omega^r_{(\tau-t_2)|t_2} \subset \Omega^r_{(\tau-t_1)|t_1} \quad \forall t_1 \le t_2 \le \tau \quad \forall r \in \mathcal{S},
\end{equation}
since these are the only dynamic components in $\mathcal{Z}^s$. Consider a forward reachable set $\Omega^r_{(\tau-t)|t}$ predicted for absolute time $\tau$ based on an initial set $\Omega^r_{0|t}$. Now consider a sensor measurement at time $t+1$ that localizes the agent to a set $\bar{\Omega}^r_{0|(t+1)}$ based on Assumption~\ref{As:Sensor}.  We know that the agent lies both in the set $\bar{\Omega}^r_{0|(t+1)}$ from the sensor, and in the one-step reachable set $\Omega^r_{1|t}$ computed from the previous time step.  Therefore, at time step $(t+1)$ the agent is in $\Omega^r_{0|(t+1)} = \bar{\Omega}^r_{0|(t+1)} \cap \Omega^r_{1|t} \subset \Omega^r_{1|t}$.  By the nested property of reachable sets in (\ref{Eq:NestedReachableSets}) it follows that $\Omega^r_{(\tau-(t+1))|(t+1)}\subset \Omega^r_{(\tau-t)|t}$ for all $t \le \tau$.  Applying mathematical induction on $t$ proves the desired result. 
\end{proof}
\end{lemma}
As an illustration of the concept formalized in Lemma~\ref{lemma1},  Fig.~\ref{fig:shrinking_set} depicts the forward reachable sets as circular shapes. The black circles represent the forward reachable sets at time step $t=0$ and the red circles show the forward reachable sets at the next time step $t=1$. By comparing the sets in absolute time $\tau = t+k$, the reachable set at prediction step $k=1$ at execution time step $t=1$ (shown in solid red) is always contained in the reachable set at prediction step $k=2$ at execution time step $t=0$ (shown in solid gray), so $\Omega^r_{\tau|t=1} \subset \Omega^r_{\tau|t=0}$. Therefore, the safe set $\mathcal{Z}^s_{\tau|t=0}$ at time $t = 0$ is always contained in the safe set at time step $t=1$. At time $t=0$, the safe set is $\mathcal{Z}^s_{\tau|t=0}$
$=(\Omega^r_{\tau|t=0} \cup \mathcal{O}^b)^c$ and at time $t=1$, the safe set is $\mathcal{Z}^s_{\tau|t=1}$
$=(\Omega^r_{\tau|t=0} \cap \Omega^r_{\tau|t=1} \cup \mathcal{O}^b )^c$. So $\mathcal{Z}^s$ is always non-decreasing.
% Therefore, the safe set $\mathcal{Z}_s$ in (\ref{eq:feasible_set_prev}) is always contained in the safe set in (\ref{eq:feasible_set_next}), so $\mathcal{Z}_s$ is always expanding.

We now state our main theoretical result.

\begin{theorem}
\label{theorem1}
If a feasible solution exists for the NMPC optimization  \eqref{eq:MPC_formulation} at time $t=0$, a feasible solution exists for all $t>0$, that is, the proposed NMPC framework is recursively feasible. 
\end{theorem}
\begin{proof}
Consider the NMPC problem \eqref{eq:MPC_formulation} is solved at time $t=0$.
% \begin{subequations}\label{eq:proof_MPC_t0}
% \begin{align}
% & \min_{\substack{\mathbf{u}_{0:N-1}}}
% & & \nonumber{\sum _{\tau = 0}^{N} l(\mathbf{z}_{\tau},\mathbf{u}_{\tau})} \\
% &\textrm{subject to} & & \nonumber{ \eqref{eq:dynamic_constraint}, \ \eqref{eq:state_input_bound}, \ \mathbf{z}_{0|t=0} = \mathbf{z}(t=0)},\\
% & & & \|\mathbf{z}_{k|t}-\mathbf{z}_{k-1|t}\|g_{\tau}(\mathbf{z}_{k|t}) \le 0,\\
% & & & {\mathbf{z}}_{N|t=0} = \mathbf{z}_{N-1|t=0},\\
% & & & \mathbf{z}_{k|t=0} \in  (\Omega_{\tau|t=0})^c \cup \mathbf{z}_{k-1|t=0}, \label{eq:feasible_set_prev} \\
% & & & \notag \nonumber \tau = t+k, \ {\textrm{and~}  k \in \{0,1,2,..,N\}}.
% \end{align} \end{subequations}
Assume feasibility of $z_0$ and let $[u_{0|0}^*, u_{1|0}^*,...,u_{N-1|0}^*]$ be the optimal feasible control sequence computed at $z_0$ and $[z_0,z_{1|0},...,z_{N|0}]$ be the corresponding feasible state trajectory (which are assumed to exist at $t=0$). Apply $u_{0|0}^*$ and let the system evolve to $z_1 = f(z_0, u_{0|0}^*)$. At $z_1$ consider the control sequence $[u_{1|0}^*,u_{2|0}^*,...,u_{N-1|0}^*, u_{N-1|0}^*]$ for the consecutive MPC problem at time $t=1$.
% \begin{subequations}\label{eq:proof_MPC_t1}
% \begin{align}
% & \min_{\substack{\mathbf{u}_{1:N}}}
% & & \nonumber{\sum _{\tau = 1}^{N+1} l(\mathbf{z}_{\tau},\mathbf{u}_{\tau})} \\
% &\textrm{subject to} & & \nonumber{ \eqref{eq:dynamic_constraint}, \ \eqref{eq:state_input_bound}, \ \mathbf{z}_{0|t=1} = \mathbf{z}(t=1)},\\
% & & & \|\mathbf{z}_{k|t}-\mathbf{z}_{k-1|t}\|g_{\tau}(\mathbf{z}_{k|t}) \le 0,\\
% & & & {\mathbf{z}}_{N+1|t=1} = \mathbf{z}_{N|t=1},\\
% & & & \mathbf{z}_{k|t=1} \in (\Omega_{\tau|t=0} \cap \Omega_{\tau|t=1} )^c \cup \mathbf{z}_{k-1|t=1}, \label{eq:feasible_set_next} \\
% & & & \notag \nonumber \tau = t+k, \ {\textrm{and~}  k \in \{0,1,2,..,N\}}.
% \end{align} \end{subequations}
Since the time-varying state safe set $\mathcal{Z}^s$ is non-decreasing (according to Lemma \ref{lemma1}) and the input constraints do not change, the state and input constraints are satisfied for all times except for potentially at state $z_{N|1}$.  However, applying $u_{N-1|1} = u_{N-1|0}^*$ at the last step keeps the state at the previous state $z_{N|1} = z_{N-1|1}$. At time $t=1$ in case the collision-avoidance constraint is violated at the last prediction time step $g(\mathbf{z}_{k=N|t=1}) > 0 $, the fact that the robot is stopped $z_{N|1} = z_{N-1|1}$ ensures the robot will not violate the complementarity constraint at the last step, so the solution remains feasible.  Applying this reasoning in a mathematical induction for all subsequent time steps yields the result stated in the theorem.  
% Since the time-varying state safe set $\mathcal{Z}_s$ is always expanding (according to Lemma \ref{lemma1}) and applying $u = u_{N-1}^*$ at the last step keeps the state at the previous state $z_{N+1} = z_{N}$.   
\begin{figure}
\centering 
\includegraphics[width=1\columnwidth]{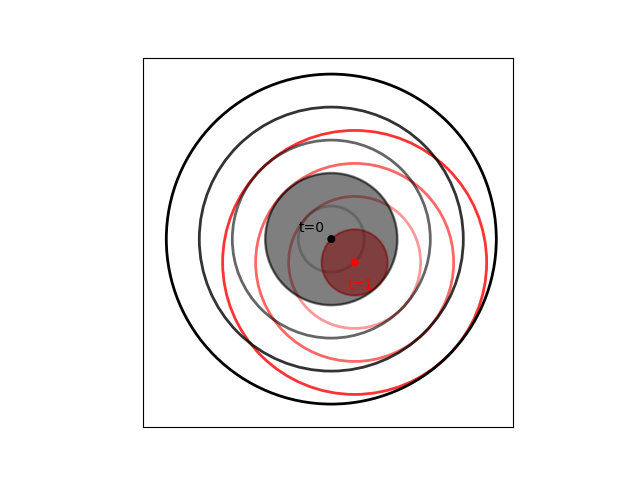}
	\caption{Forward reachable sets of unsafe region $\Omega$ are shown for two consecutive time steps $t=0$ and $t=1$ in the simple case of a single pedestrian based on a top speed bound.} 
	\label{fig:shrinking_set}
\end{figure}
\end{proof}
While the result applies generally to all problems that fall under our assumptions, in the rest of this paper we consider a specific setup in which the sensor is a LiDAR sensor, and the agents are assumed to be able to move in any direction with a known speed bound, leading to a simple geometric computation of reachable sets.  We develop a practical algorithm based on (\ref{eq:MPC_formulation}) for this setup.
\section{Iterative Optimization Scheme}
\label{SEC6}
\subsection{NMPC Optimization}
The following optimization represents the nonlinear MPC scheme. The collision-avoidance with static obstacles in the environment $\mathcal{O}^b$ represented in \eqref{CA_static_env} are performed based on point-cloud data received from LiDAR. (Other sensors and methods can be used for collision detection and avoidance as well). Figure \ref{fig:CA_PCL} depicts a LiDAR-based collision avoidance scheme, where black dots are point clouds and blue areas are obstacles. Point clouds are down-sampled and green circles are centered at the sampled point clouds. Circles' radius should be selected to ensure full coverage of the visible portion of the obstacle. The safe distance between the robot (magenta circle) and each green circle is incorporated as a constraint in the NMPC optimization
\begin{subequations}\label{eq:NMPC}
\begin{align}
& \min_{\substack{\mathbf{u}_{t:t+N-1}}}
& &\eqref{eq:total_cost} \\
&\textrm{subject to} & & \eqref{eq:dynamic_constraint}, \eqref{eq:initial_cond},\eqref{eq:state_input_bound}, \eqref{CA_static_env}, \ \eqref{eq:complementarity_constraint}, \ \eqref{eq:final_constraint}\\
\label{obstacle_avoidance_constraint2}
& & & \textrm{dist}(\mathcal{C}(\mathbf{z}_{k|t}), \bar{\mathbf{z}}_{\textrm{proj}{k|t}}) \geq d_{\textrm{safe}}, \ r \in \mathcal{S},\\
& & & \notag \nonumber{\textrm{and~}  k \in \{0,1,2,..,N\}},
\end{align} \end{subequations}
where $\bar{{\mathbf{z}}}_{\textrm{proj}k|t}^r$ denotes the projected state trajectory computed by solving collision avoidance optimization (The bar notation means the value is known). More details are included in the next section. The computed minimum distance is then constrained to be larger than a safe predefined minimum allowable distance $d_{\textrm{safe}}$. ($d_{\textrm{safe}}$ is a design parameter and should be determined based on the uncertainty quantification of physical models and stochastic measurement errors.)
% \begin{figure}
% \centering
% \begin{minipage}{.5\linewidth}
%   \centering
%   \includegraphics[width=1\linewidth]{Pics/CA_PCL.png}
%   \captionof{figure}{LiDAR-based collision avoidance: blue areas are obstacles and black dots are point clouds. Green circles are centered at down-sampled point-clouds.}
%   \label{fig:CA_PCL}
% \end{minipage}%
% \begin{minipage}{.5\linewidth}
%   \centering
%   \includegraphics[width=1\linewidth]{Pics/Lidar_PCL.png}
%   \captionof{figure}{Boundary of occlusion is detected based on the jump in consecutive range values of point 1 and point 2. }
%   \label{fig:Lidar_PCL}
% \end{minipage}
% \end{figure}
% \begin{figure}
% \centering
% \includegraphics[width=\linewidth]{Pics/CA_LiDAR2.png}
% \caption{\textbf{Left: }LiDAR-based collision avoidance: Blue areas are obstacles; Black dots are point clouds; Green circles are centered at down-sampled point-clouds. \textbf{Right: }Boundary of occlusion is detected based on a jump in range values of two consecutive points 1 and 2.}
% \label{fig:CA_PCL2s}
% \end{figure}

\begin{figure}
\centering
\includegraphics[width=0.7\linewidth]{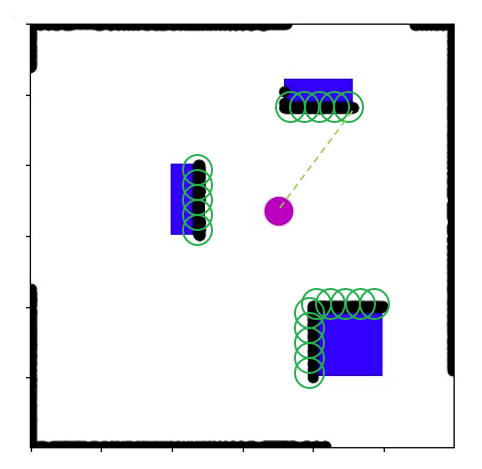}
\caption{LiDAR-based collision avoidance: Blue areas are obstacles; Black dots are point clouds; Green circles are centered at down-sampled point clouds.}
\label{fig:CA_PCL}
\end{figure}

\begin{figure}
\centering
\includegraphics[width=0.8\linewidth]{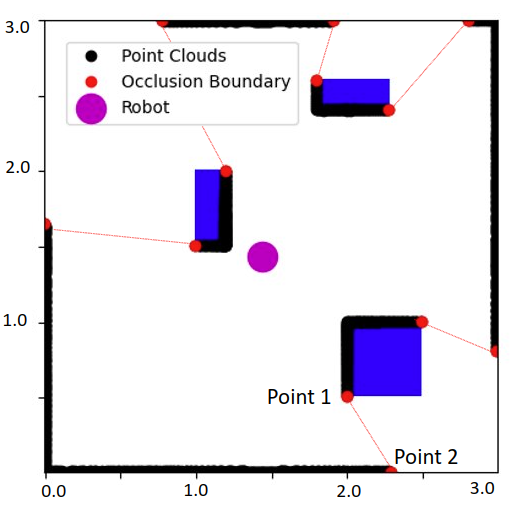}
\caption{Boundary of occlusion is detected based on a jump in range values of two consecutive points 1 and 2.}
\label{fig:Lidar_PCL}
\end{figure}

\subsection{Collision Avoidance Optimization}
\emph{Occlusion detection:} To detect the boundary of occlusion, the consecutive range values obtained from LiDAR are compared, and the ones with a difference larger than a threshold are detected as occlusion boundaries. For example, in Figure \ref{fig:Lidar_PCL}, there is a large jump between the range values of Point 1 and Point 2 (red points), so the LiDAR line-of-sight that passes through those two points indicates a boundary of occlusion. All the other red points in the figure are computed in the same way and their line-of-sight indicates an occlusion boundary. \\
\emph{Reachable set construction:} 
The reachable set represents the potential locations where target agents (pedestrians) could be present.
The initial set for constructing the reachable set is a line segment which is the boundary of occlusion detected by LiDAR. 
%The parts that define the occlusion region are all line segments, so a safe upper-bound on reachable set is created for any target that might present in these occluded regions. 
To consider the worst-case, the pedestrian is assumed to start its movement from the occlusion boundary. No prior knowledge of the agents' dynamic model is assumed and the only assumption is that the target maximum speed $v_{\textrm{target}}$ is known. This upper-bound speed can be specified according to the environment such as driving or in-door settings.
% For example, in a autonomous driving setting, a possible moving target can be another vehicle, so the speed upper-bound is larger compared to robot navigation in indoor environment, where the agents will be pedestrians. Since upper-bound on speed is assumed and the target can move at any direction, the reachable sets will be capsule sets.

Capsules are described as the union of two circles $\mathcal{C}^1$, $\mathcal{C}^2$, and a rectangle $\mathcal{P}$ as shown in Figure \ref{fig:capsule_construction}. Figure \ref{fig:capsules_sets} shows the forward reachable sets that are nested capsules, which grow over time along the MPC horizon. The reachable sets are enlarged at each time step $k$ along the MPC planning horizon. The enlarged capsule set is computed by increasing the radius of the two circles $\mathcal{C}^1$ and $\mathcal{C}^2$ by distance $d_{\textrm{target}} = \frac{v_{\textrm{target}}}{\Delta t}$, which is the maximum distance that a pedestrian can travel in a time step $\Delta t$. The corresponding polytopic set is constructed as shown in Figure \ref{fig:capsule_construction} based on the convex hull of the 4 points (green dots) obtained from circle enlargement. The capsule set construction procedure is summarized in Algorithm \ref{alg_capsule_set}. \\
\begin{figure}
\centering
\includegraphics[width=0.5\columnwidth]{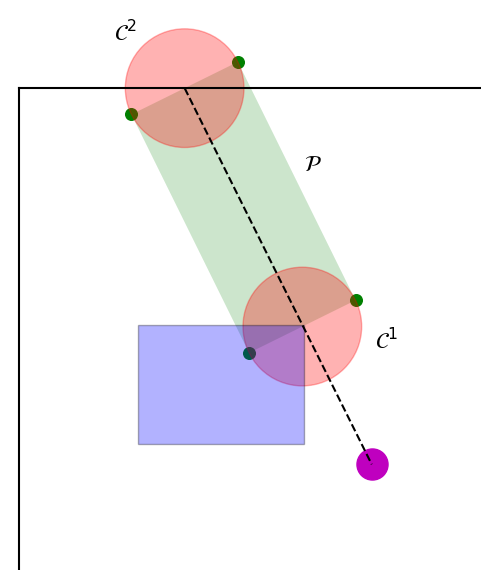}
\caption{The capsule set is constructed by polytope and circles.}
\label{fig:capsule_construction}
\end{figure}
\begin{figure}
\centering
\includegraphics[width=0.5\textwidth]{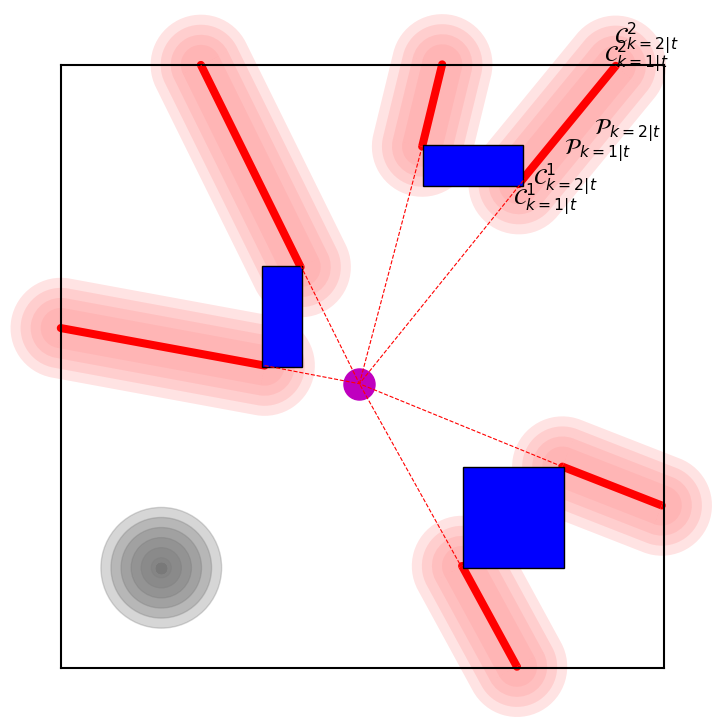}
\caption{Capsule sets are enlarged along the MPC prediction horizon based on the speed upper-bound $v_{\textrm{target}}$ assumed for the invisible agent. The gray circles illustrate forward reachable sets for visible targets. The red dashed lines represent LiDAR rays.} 
\label{fig:capsules_sets}
\end{figure}
\begin{algorithm}[t]
\caption{\small Capsule Set Construction Algorithm}\label{alg_capsule_set}
    \begin{algorithmic}[1]
        \State According to the LiDAR scan of the environment, the jump in the values of two consecutive range measurements indicates the boundary of an occluded region. 
        \State {The circles $\mathcal{C}^1$ and $\mathcal{C}^2$ with radius $d_{\textrm{target}}$ centered at the endpoints of line-of-sight are created and the vertices of the associated polytope are determined according to the circles' radius.}
        \State {Convex hull of polytope vertices is obtained to determine H-representation of the polytope, necessary for collision avoidance optimization.}
    \end{algorithmic}
\end{algorithm}
\emph{Minimum distance from the capsule set:} The robot shape is approximated as a circle with radius $r_{\textrm{robot}}$ and the collision avoidance between the robot and the dynamic agent such as pedestrian is formulated based on the minimum distance between the circle (robot) and the capsules (the pedestrian forward reachable sets). To compute the minimum distance $\textrm{dist}(\mathcal{C}_{\textrm{robot}},\Omega)$, first the robot's center of mass $z$ is projected on to the capsule $ d_{\textrm{proj}} = \textrm{min}\{ d_{\mathcal{C}^1}, d_{\mathcal{C}^2}, d_{\mathcal{P}}\},$
where $d_{\mathcal{C}^1}, d_{\mathcal{C}^2}, d_{\mathcal{P}}$ are distances to the circle $\mathcal{C}^1$ and circle $\mathcal{C}^2$ and the polytope $\mathcal{P}$, respectively. Then the robot radius is subtracted as $\textrm{dist}(\mathcal{C}_{\textrm{robot}}(\mathbf{z}),\Omega) = d_{\textrm{proj}} - r_{\textrm{robot}}$. Distance to the circles $d_{\mathcal{C}^1}$ and $d_{\mathcal{C}^2}$ are computed by calculating the distance to the center of circles and subtracting the radius of the circle. Distance to the polytope $\mathcal{P}$ is computed by solving the optimization problem: $d_{\mathcal{P}} = \underset{\mathbf{y}}{\min}\{\|\bar{\mathbf{z}}-\mathbf{y}\|_{2}|\mathbf{A}\mathbf{y} \leq \mathbf{b}\},$ where $\bar{\mathbf{z}}$ is the robot position, which is known (the bar notation means the value is kept as fixed). The optimal solution of the above optimization problem is denoted as $\mathbf{y}^*$. 
After computing $d_{\textrm{proj}}$, the corresponding projected point $\mathbf{z}_{\textrm{proj}}$ on the capsule set is computed (the projection of robot on the capsule set) by 
\begin{equation} \label{eq:CA}
 \mathbf{z}_{\textrm{proj}} =\begin{cases}
  \mathbf{y}^*  & \textrm{if} \ d_{\textrm{proj}} = d_{\mathcal{P}} \\
(1-\textrm{ratio}_1) \bar{\mathbf{z}} + \textrm{ratio}_1 \mathcal{C}^1_{\textrm{center}}& \textrm{if}\ d_{\textrm{proj}} = d_{\mathcal{C}_1}  \\
(1-\textrm{ratio}_2) \bar{\mathbf{z}} + \textrm{ratio}_2 \mathcal{C}^2_{\textrm{center}}     & \textrm{if}\ d_{\textrm{proj}} = d_{\mathcal{C}^2},
\end{cases}
\end{equation}where $\textrm{ratio}_1 = (d_{\mathcal{C}^1} - r_{\mathcal{C}^1}) / d_{\mathcal{C}^1}$ and $\textrm{ratio}_2 = (d_{\mathcal{C}^2} - r_{\mathcal{C}^2}) / d_{\mathcal{C}^2}$, and $r_{\mathcal{C}^1}$ and $r_{\mathcal{C}^2}$ are radius of the Circle 1 and Circle 2, respectively. In \eqref{eq:CA}, the first equation is the projection of the robot on polytope and the second and third equations represent the projection of the robot on Circle 1 and Circle 2, respectively.
The projection of the robot on the capsule set $\mathbf{z}_{\textrm{proj}}$ is computed by solving \eqref{eq:CA}. This projection step is solved for all the robot predicted states along the MPC horizon $\bar{\mathbf{z}}_{k|t} \ \forall k \in \{1,...,N\}$ (open-loop state trajectory) and for each boundary of occlusion $r \in \mathcal{S}$, in parallel. The projected open-loop state trajectory $\bar{{\mathbf{z}}}_{\textrm{proj}k|t}^r$ is computed and is incorporated into the constraint \eqref{obstacle_avoidance_constraint2} as a fixed known value. To summarize, the optimization \eqref{eq:NMPC} and optimization \eqref{eq:CA} alternate. In the NMPC \eqref{eq:NMPC}, the projected states $\bar{\mathbf{z}}^r_{\textrm{proj}k|t}$ are kept fixed and decision variables of optimization are the robot states along the MPC horizon $\mathbf{z}_{k|t}$. In collision avoidance optimization \eqref{eq:CA}, the robot states are kept fixed $\bar{\mathbf{z}}_{k|t}$ and the projected states ${\mathbf{z}}_{\textrm{proj}k|t}$ are computed.\\ 
Also, to use open-loop projected states in the NMPC in constraint (4c), we need to rely on the projected states from the previous iteration $(t-1)$ of NMPC solutions (open-loop state trajectory). So, before solving \eqref{eq:CA}, the robot state trajectory $\bar{\mathbf{z}}_{k|t}$ is shifted one step forward in time and the last step is {extrapolated} as $[\bar{\mathbf{z}}(2),...,\bar{\mathbf{z}}(N),\bar{\mathbf{z}}(\textrm{Extrap})]$.

Note that the optimization problem for projecting the robot on polytope and finding $d_{\mathcal{P}}$ has a lower bound on the minimum distance which is zero, so in case the point $\mathbf{z}$ is already inside the polytopic set $\mathcal{P}$ the minimum distance is calculated as zero. One alternative solution is to formulate the collision avoidance with capsule sets exactly the same as collision avoidance formulation for static obstacles, in which circles are sampled on the occlusion boundary and the robot distance to these circles is incorporated into the NMPC problem directly as collision avoidance constraint. Using this alternative approach no iterative optimization scheme is required. However, incorporating the constraints directly into the NMPC can slightly affect the computation time. 
\subsection{Iterative Optimization Algorithm}
Problem \eqref{eq:CA} is itself an optimization problem, so imposing it as a constraint of Problem \eqref{eq:MPC_formulation} yields a bilevel optimization that is computationally intensive. To overcome this challenge, by relying on the ability of MPC to generate predicted open-loop trajectories, we devise an alternative optimization algorithm in which we alternate between the two optimizations: NMPC optimization \eqref{eq:NMPC} and collision avoidance optimization \eqref{eq:CA}, as detailed in Algorithm~\ref{alg_alternation}. 
\begin{algorithm}[t]
\caption{\small Iterative Optimization Algorithm}\label{alg_alternation}
    \begin{algorithmic}[1]
        \State Initialize state trajectory $[\mathbf{z}(1),...,\mathbf{z}(N)]$ 
        % based on the current LiDAR measurements,
        \For { $t=0,1,..., \infty$} 
            \State  {Receive LiDAR scan and compute the capsule sets: $[\Omega^r(1),...,\Omega^r(N)]$ for each $r \in \mathcal{S}$}
                \State {Compute the shifted state $[\bar{\mathbf{z}}(2),...,\bar{\mathbf{z}}(N),\bar{\mathbf{z}}(\textrm{Extrap})]$.}
            \State {Solve Problem \eqref{eq:CA} for the shifted state trajectory for each $r \in \mathcal{S}$ in parallel to find projected robot states on the capsule sets $[\mathbf{z}_{\textrm{proj}}(2),..., \mathbf{z}_{\textrm{proj}}(N), \mathbf{z}_{\textrm{proj}}(\textrm{Extrap})]$ for each prediction step of the open-loop trajectory}.
          \State {Solve NMPC Problem~\eqref{eq:NMPC}}
        % \hphantom{zzzzzzz}

        %:        $[\mathbf{s}_{ij}(1),...,\mathbf{s}_{ij}(N)]$, $[\boldsymbol{\lambda}_{ij}(1),...,\boldsymbol{\lambda}_{ij}(N)],$ $[\boldsymbol{\lambda}_{ji}(1),...,\boldsymbol{\lambda}_{ji}(N)].$ 
        \State {Apply $\mathbf{u}_{\text{MPC}}$ to move forward.}
        \EndFor
    \end{algorithmic}
\end{algorithm}
In step \textcircled{\footnotesize{1}} of Algorithm~\ref{alg_alternation}, the sequence of state trajectory is initialized. In step \textcircled{\footnotesize{3}}, a new LiDAR measurement is received and the reachable sets are computed accordingly. In step \textcircled{\footnotesize{4}}, the associated open-loop state trajectory is shifted forward one step in time, and the last step is copied or extrapolated. In step \textcircled{\footnotesize{5}}, the collision avoidance optimization is solved in parallel for each boundary of the occluded area. In step \textcircled{\footnotesize{6}} NMPC problem \eqref{eq:NMPC} is solved and a sequence of open-loop actions is computed. In step \textcircled{\footnotesize{7}}, the first control input $\mathbf{u}_{\textrm{MPC}}$ is applied and the procedure is repeated for the next time step.

\section{Simulations and Hardware Experiments} 
\label{SEC7}
To validate the effectiveness of the proposed approach, both simulation studies and real-world experiments have been performed. The NMPC optimization is modeled using CasADi \cite{andersson_gillis_diehl_2017} in Python. For the experiment, a TurtleBot robot equipped with Velodyne LiDAR (Vlp-16 Puck Lidar Sensor 360 Degree) is used. OptiTrack motion capture system is used for robot localization. The MPC is run on a Lenovo ThinkPad laptop with Ubuntu 20.04 OS with 3.00 GHz Intel CPU Core i7 with 32 GB of RAM.

The Robot dynamic is modeled by a nonlinear kinematic unicycle model $\dot x = v \cos(\psi),\
\dot y  = v \sin(\psi), \
\dot \psi = \delta,$
% \begin{equation}\label{eq:unicycle_model}
% \begin{aligned}
% \dot x & = v \cos(\psi),\
% \dot y  = v \sin(\psi), \
% \dot \psi = \delta,\\
% \end{aligned}
% \end{equation}
where the state vector is $\mathbf{z} = [x,y,\psi]^\top$ ($x$, $y$, and $\psi$ are the longitudinal position, the lateral position, and the heading angle, respectively). The input vector is $\mathbf{u} = [v, \delta]^\top$ ($v$ and $\delta$ are the velocity and the steering angle, respectively). Using Euler discretization, the unicycle model is discretized with sampling time $\Delta t$ as $x(t+1) = x(t) + \Delta t \ v(t) \cos(\psi(t)), \
 y(t+1)  = y(t)+\Delta t \ v(t)  \sin(\psi(t)),\ 
\psi(t+1) = \psi(t)+ \Delta t \ \delta(t).$
% \begin{equation}\label{eq:unicycle_model_discrete}
% \begin{aligned}
% x(t+1) & = x(t) + \Delta t \ v(t) \cos(\psi(t)), \\
% y(t+1) & = y(t)+\Delta t \ v(t)  \sin(\psi(t)), \\
% \psi(t+1) &= \psi(t)+ \Delta t \ \delta(t),\\
% \end{aligned}
% \end{equation}
The cost is defined as $l(\mathbf{z},\mathbf{u})$ $= \sum_{k=t}^{t+N}$ $\|(\mathbf{z}_{k|t}$ $-\mathbf{z}_\textrm{Goal})\|^2_{Q_\mathbf{z}}
$ $+ \sum_{k=t}^{t+N-1}$ $(\|(\mathbf{u}_{k|t})||^2_{Q_\mathbf{u}} $ $+ \|$ $(\Delta \mathbf{u}_{k|t})\|^2_{Q_{\boldsymbol{\Delta u}}})$, where $\Delta \mathbf{u}$ penalizes changes in the input rate.
$Q_\mathbf{z}\succeq 0$, $Q{_\mathbf{u}}, Q_{\boldsymbol{\Delta u}}\succ 0$ are weighting matrices, $\mathbf{z}_\textrm{Goal}$ is the goal location that robot should reach. Throughout the simulations the sampling time $\Delta t$ is 0.1s. Simulation studies are performed for two different scenarios including navigating a corner and navigating a complex environment with multiple obstacles. In both scenarios, the proposed Occlusion-Aware planner (OA-MPC) is compared against baseline MPC which is agnostic towards occlusions and the possible presence of invisible agents.

Figure \ref{fig:corner_plot} illustrates the corner navigation scenario. The closed-loop simulation for both planners (OA-MPC and baseline MPC) is shown.  
\begin{figure}[!htb]
\centering
  \includegraphics[width=1\columnwidth]{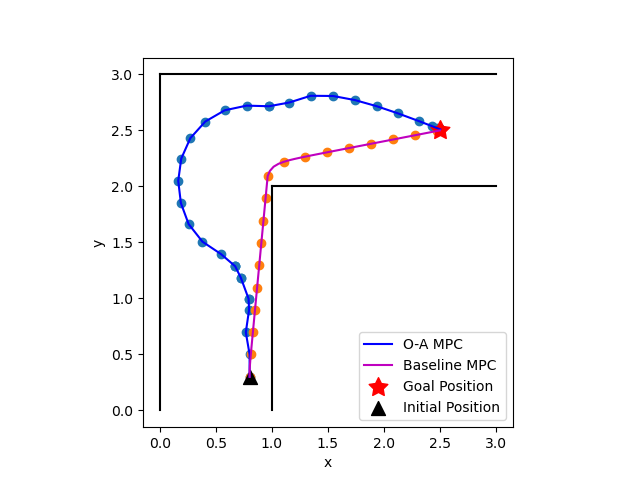}
  \caption{Corner navigation simulation.}
  \label{fig:corner_plot}
\end{figure}
% Closed-loop simulation of navigating corner by assuming the presence of a hidden dynamic agent in the occluded region.
In this simulation, minimum allowable distance $d_{\textrm{safe}}$ is $0.5 m$, the velocity input is bounded within $0$ to $2 m/s$. The steering input is bounded within $\pm\, \pi rad$. The robot shape is defined by a circular shape with radius $r_\textrm{robot} = 0.2 m$. Boundaries of the track shown are defined as the limits on the position states $x$ and $y$. The width of the track is $1m$. The robot's initial condition $\mathbf{z}_0$ is $[x_0, y_0, {\psi}_0]$ = $[0.8 m,0.3 m, \frac{\pi}{2} rad]$. Two goal positions $\mathbf{{z_{goal}}_1}$ as $[x=1,y=2.5]$ and $\mathbf{{z_{goal}}_2}$ as $[x=2.5,y=2.5]$ are specified. The MPC prediction horizon $N $ is $10$. A potential dynamic agent hidden behind the occluded area is assumed to be moving with speed $v_{\textrm{ped}}=0.5 m/s$. As seen, baseline MPC (red) is agnostic towards the occlusion boundary and navigates the corner without considering collision with any possible potential invisible agent that might emerge from the occluded region. Conversely, the OA-MPC planner (blue) slightly turns to the left as it approaches the corner to take a wider turn before turning completely to the right to avoid collision with any potential invisible agent.
% Figure \ref{fig:open-loop} shows MPC open-loop trajectories and how they evolve as the robot avoids the enlarged capsule sets along the MPC horizon.  
% shows the plots the closed-loop state and input trajectories for the same corner navigation scenario. For example as seen, the heading angle $\psi$ changes near the corner as the robot slightly turns to the left before turning completely to the right. The bottom center shows the distance between the robot and the reachable sets (collision avoidance sets) during the closed-loop simulation. (The distance is calculated at MPC execution time $t$ and not the predictions along the horizon $k=0$.

Figure \ref{fig:complex_map_plot_comparison} shows navigation in an environment with multiple obstacles that lead to various occluded regions. Closed-loop simulation for baseline MPC and OA-MPC with different speeds assumed for invisible agents are compared. As seen, the baseline MPC (blue) is the less conservative planner which is close to the obstacles, but the OA-MPC with $v_{\textrm{ped}} = 0.3 m/s$ considers collision avoidance with the possible invisible agent and keeps a larger distance to the obstacles. Also, as the presumed speed upper-bound for the invisible target agent increases to $v_{\textrm{ped}} = 0.5 m/s$ and $v_{\textrm{ped}} = 0.6 m/s$ values (shown in red and light green, respectively), the robot's behavior is more conservative and it keeps larger distance towards the occluded regions, as expected. Computation times for baseline MPC and OA-MPC for the scenario of Figure \ref{fig:complex_map_plot_comparison} are reported in Table \ref{table:1} for over 100 iterations. As seen, the average computation time for OA-MPC is larger than baseline MPC, due to extra constraints associated with reachable sets on occluded boundaries, but OA-MPC is still suitable for real-time implementation.
% \begin{figure}[!htb]
% \centering
%   \includegraphics[width=0.8\linewidth]{Pics/complex_map_comparison.png}
%   \caption{Navigation in environment with multiple obstacles.}
%   \label{fig:complex_map_plot}
% \end{figure}
% Closed-loop simulation of navigating complex environment with multiple obstacles and therefore more occluded regions.
Figure \ref{fig:OL_traj} shows the enlargement of capsule sets along the MPC horizon. The blue trajectory shows the closed-loop plan from the initial condition $\mathbf{z}_0 = [x_0, y_0, {\psi}_0]$ = $[0.8 m,0.3 m, \frac{\pi}{2} rad]$ to the goal position $\mathbf{{z_{goal}}}$ as $[x=2.9,y=2.8]$. The open-loop trajectory at time step $t=7 s$ is shown as colored dots, where each color represents one step of prediction. For example red dot shows $\mathbf{z}_{k=0|t=7}$, light red is $\mathbf{z}_{k=1|t=7}$, purple is $\mathbf{z}_{k=2|t=7}$, green is $\mathbf{z}_{k=3|t=7}$, yellow is $\mathbf{z}_{k=4|t=7}$, light blue is $\mathbf{z}_{k=5|t=7}$, pink is $\mathbf{z}_{k=6|t=7}$ and so on (only portion of horizon steps are illustrated in the figure for clarity). 
The corresponding capsule sets are shown in the same color. As seen, at each horizon step $k$, the robot satisfies the minimum distance requirement to capsule sets. 
% Note that OA-MPC involves two optimizations, Collision Avoidance (CA) optimization is a simple projection step that is a quadratic problem and can be solved in real-time and in parallel for each occlusion boundary $r \in \mathcal{S}$, but the NMPC optimization computation time can increase in very cluttered environments. One way to decrease the computation time is to include the occlusion boundaries in the robot path and remove the constraints that are not behind the robot and are not required to be included in planning. }

\begin{table}
\centering
\begin{tabular}{ |p{2.5cm}|p{1cm}|p{1cm}|p{1cm}| }
\hline
\multicolumn{4}{|c|}{ Computation Time (Sec)} \\
\hline
Approach & Average & Max & Std \\
\hline
Baseline MPC &0.023& 0.030 & 0.004 \\
\hline
NMPC Optimization  &0.033 & 0.054 & 0.006 \\
CA Optimization  &0.004 & 0.007 & 0.001\\
OA-MPC (Total)  &0.037 &0.061 & 0.007 \\
\hline
\end{tabular}
\caption{Computation time statistics over 100 iterations is compared for baseline MPC and OA-MPC.}
\label{table:1}
\end{table}

\begin{figure}
\centering 
\includegraphics[width=1\linewidth]{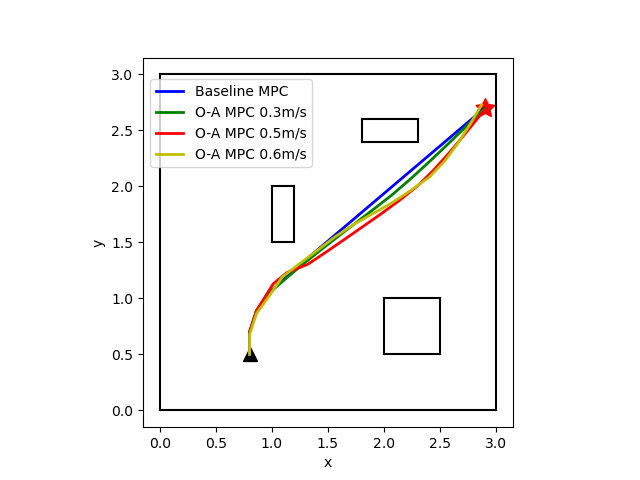}
	\caption{Navigation in the environment with multiple obstacles. The plots compare robot-executed trajectories for various speed upper bounds assumed for the target agent.} 
	\label{fig:complex_map_plot_comparison}
\end{figure}

\begin{figure}
\centering 
\includegraphics[width=0.8\linewidth]{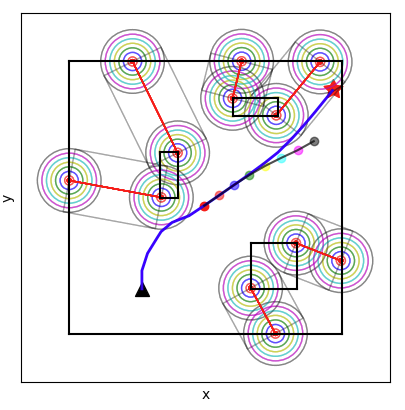}
	\caption{Open-loop MPC trajectories are shown for time step $t=7 s$. Capsule sets corresponding to each predicted step of  MPC are shown in the same color.}
	\label{fig:OL_traj}
\end{figure}

Figure \ref{fig:ped_sim0} shows snapshots of closed-loop simulation of LiDAR and OA-MPC for a scenario in which a robot (magenta circle) navigates in an environment with multiple static (blue regions) and dynamic obstacles or pedestrians (green circles). The speed limit of the robot in this scenario is considered as $10m/s$. The Top (a) figure shows the initial condition in which one pedestrian is visible to the robot and the other is initially hidden in an occluded region. Point clouds from the LiDAR simulation are shown in purple. The boundary of occlusion and capsule reachable sets are shown in red. As the robot navigates the environment to reach its goal (red star), the occlusion boundary changes as the LiDAR scan is updated. As soon as the pedestrian crosses the boundary of occlusion as shown in the Second Top (b) figure, the OA planner creates the pedestrian's reachable sets (nested green circles) and avoids collision with visible agents (green reachable sets) and capsules (red reachable sets) corresponding to occluded area. The Third Top (c) figure shows the open-loop trajectory (black dots) as the robot moves towards its goal. The Bottom (d) figure shows the entire executed trajectory in which the robot reaches its goal while avoiding collision with hidden/visible pedestrians. Note that visible targets can create dynamic occlusion in the environment. This type of occlusion can be handled in a similar way. However, in this simulation, the occlusion caused by a moving visible pedestrian is considered negligible due to the small size of the pedestrian. 
% \begin{figure}
% \centering 
% \includegraphics[width=1\linewidth]{Pics/scenraio_8_pics.png}
% 	\caption{\textbf{Top Left: }One of the pedestrians (green circles in the top right) is moving to the left and is visible to the robot. The other pedestrian (green circle in the bottom left region) is initially in the occluded region and she/he starts to move towards the top. \textbf{Top Right: }At the moment that the pedestrian crosses the occlusion boundary, she/he is visible to the robot, and the forward reachable sets (green) are computed to avoid collision with the visible target. \textbf{Bottom Left: }The robot moves towards its goal while avoiding collision with visible/hidden agents, (open-loop trajectories are shown in black). \textbf{Bottom Right: }Finally, the robot reaches its goal, and its executed trajectory is illustrated (blue).} 
% 	\label{fig:ped_sim0}
% \end{figure}

\begin{figure}
\centering
\begin{minipage}{\columnwidth}
    \includegraphics[width=0.95\columnwidth]{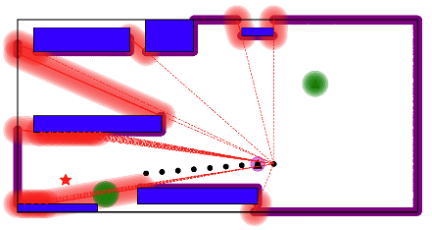}\\
     \vspace{2pt}
    \centering(a)
\includegraphics[width=0.95\columnwidth]{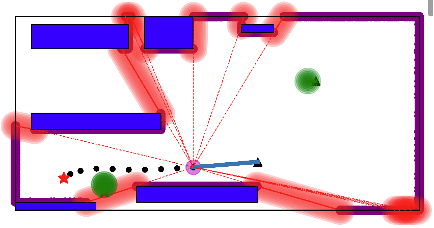}\\
     \vspace{2pt}
    \centering(b)
\includegraphics[width=0.95\columnwidth]{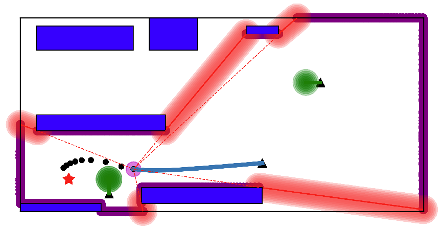}
     \vspace{2pt}
    \centering(c)
\includegraphics[width=0.95\columnwidth]{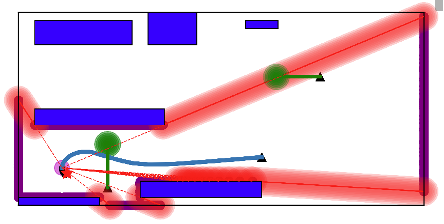}
     \vspace{2pt}
    \centering(d)
\end{minipage}
\caption{\textbf{Top (a): }One of the pedestrians (green circles in the top right) is moving to the left and is visible to the robot. The other pedestrian (green circle in the bottom left region) is initially in the occluded region and she/he starts to move towards the top. \textbf{Second Top (b): }At the moment that the pedestrian crosses the occlusion boundary, she/he is visible to the robot, and the forward reachable sets (green) are computed to avoid collision with the visible target. \textbf{Third Top (c): }The robot moves towards its goal while avoiding collision with visible/hidden agents, (open-loop trajectories are shown in black). \textbf{Bottom (d): }Finally, the robot reaches its goal, and its executed trajectory is illustrated (blue).}
\label{fig:ped_sim0}
\end{figure}

\begin{figure*}
    \centering
    \includegraphics[width=1\textwidth]{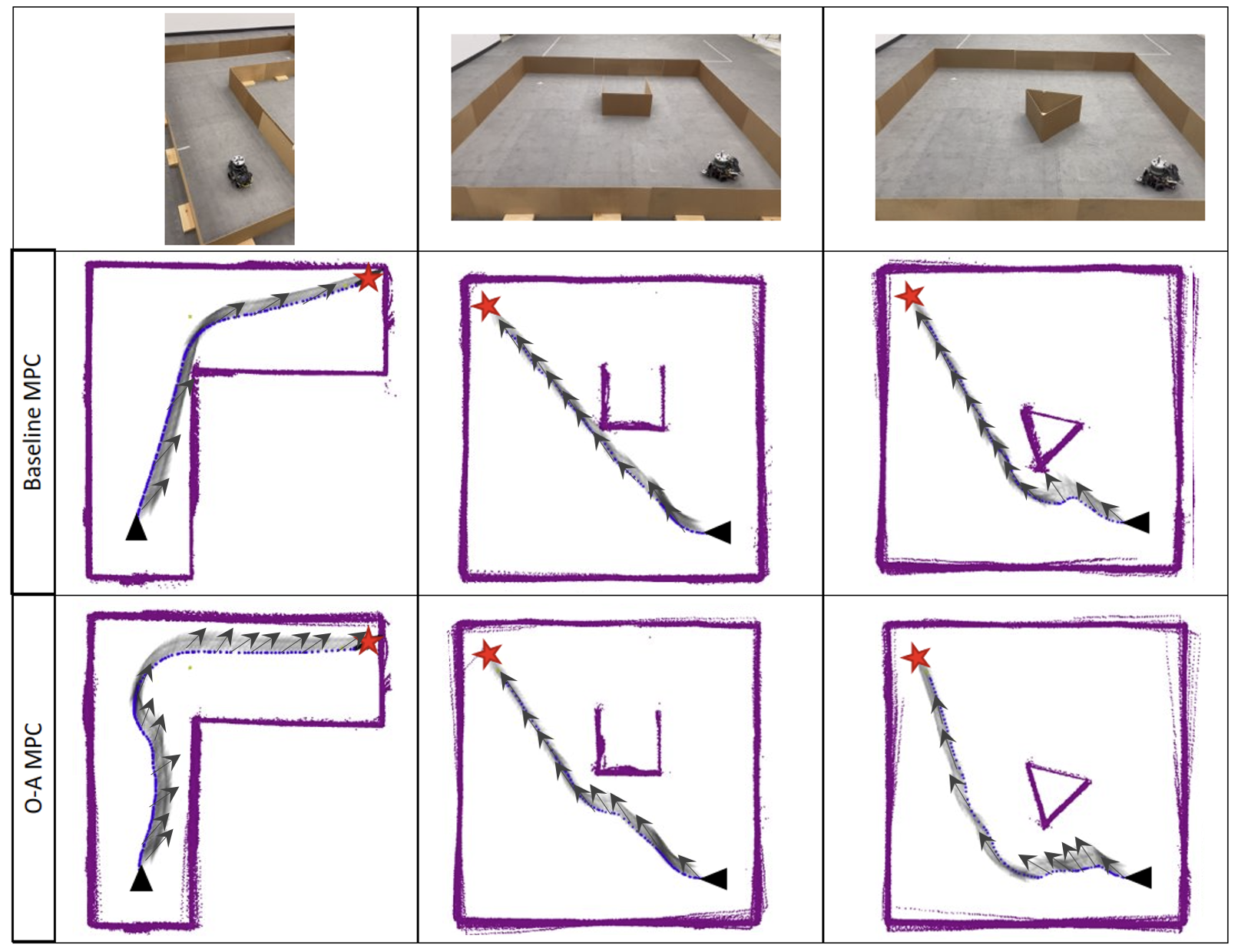}
    \caption[Short caption]{Experimental Results: Gray arrows show the robot's heading at each time step; Blue dots show the executed plan; LiDAR point clouds are shown in purple. \textbf{Top Row}: Different environments (corner, square obstacle, triangle obstacle (from left to right). \textbf{Second Row}: Baseline MPC is occlusion-agnostic and navigates the occluded regions by taking tight turns without considering collision avoidance with potential invisible dynamic agents. \textbf{Third Row}: Occlusion-Aware MPC planner takes wide turns in occluded regions.}
    \label{fig:EXP_results_2}    
\end{figure*}
Figure \ref{fig:EXP_results_2}, shows experimental results that compare baseline MPC with OA-MPC for robot navigation in different environments, a corner, square-shaped obstacle, and triangular-shaped obstacle. For all these experiments, the environment is unknown to the robot, so it relies on its LiDAR to perceive the environment in real-time. LiDAR point cloud data (purple) define the environment boundaries and obstacles.
% {\color{blue}For ease of computation, the point cloud is compressed into a 2D scan of radius and angle measurements generated via a simple projection method that rejects points outside the sensor's minimum and maximum range.}
% At each time step, the scan is aligned with respect to the global frame defined at the experiment start using robust ICP.
Both MPC planners detect the obstacles using LiDAR in real time and perform collision avoidance with obstacles.
However, they act differently in how they navigate the occluded regions. Blue dots represent the executed plan starting from the small black triangle and reaching the goal position (red star). Gray arrows show the heading of the robot at each time step. Both MPC planners operate at 10 Hz with a horizon of $N=10$. The left column shows a corner scenario in which OA-MPC takes a wider turn to navigate the corner to avoid collision with possible invisible agents, but baseline MPC which is agnostic towards occlusion regions navigates the corner without considering collision avoidance with occluded dynamic agents. The middle and right columns confirm the same results for navigation in environments with square-shaped and triangular-shaped obstacles.

Figures \ref{fig:EXP_dynamic_agent_baseline}, \ref{fig:EXP_dynamic_agent_OA_MPC}, and \ref{fig:EXP_dynamic_agent_comparison} represent experimental results that include a dynamic agent in the scene. Figure \ref{fig:EXP_dynamic_agent_baseline} illustrates the experimental results for the baseline MPC, where the robot operates in an environment with the presence of another dynamic agent. In the leftmost image labeled as 1, the ego robot is on the right side, and the dynamic agent is on the left. The ego robot plans to navigate around the corner to reach its goal. However, the baseline MPC is unaware of occlusions and does not account for potential dynamic agents in the occluded region. As shown in the middle image labeled as 2, this results in a collision. MPC planner operates at 10 Hz with a horizon of $N=5$. In the rightmost image, the ego robot's closed-loop trajectory is displayed in blue, while the dynamic agent's executed trajectory is shown in green. The goal positions for the ego robot and the dynamic agent are represented by red and green stars, respectively, and their initial positions are marked by black triangles. Lidar point clouds are shown in purple, with the green and blue circles indicating the points of collision between the robots.

Figure \ref{fig:EXP_dynamic_agent_OA_MPC} represents the experimental results for the OA-MPC, where the robot operates in the same corner environment with the presence of the same dynamic agent. The setup is the same as the previous baseline MPC and the OA-MPC planner operates at 10 Hz with a horizon of $N=5$. Images labeled 1,2,3,4, and 5 are sequential snapshots in time. The bottom right image shows the executed trajectories. The ego robot's closed-loop trajectory is displayed in blue, while the dynamic agent's executed trajectory is shown in green. The goal positions for the ego robot and the dynamic agent are represented by red and green stars, respectively, and their initial positions are marked by black triangles. Lidar point clouds are shown in purple. As seen in the snapshots and the trajectories, the OA-MPC takes a wider turn to navigate the corner to avoid collision with possible invisible agents. As expected, there is no collision, and both the ego robot and dynamic agent can safely reach their goals. Figure \ref{fig:EXP_dynamic_agent_comparison} compares the results of the OA-MPC versus the baseline MPC for the same corner scenario with a dynamic agent. As shown, OA-MPC is safe, it is aware of occluded areas and takes a wide turn close to the corner. So both the ego robot and the dynamic agent navigate the corner safely and reach their goals. However, baseline MPC which is agnostic to occlusion, does not take into account the potential invisible dynamic agent and takes a tight turn at the corner. So it results in a collision between the ego robot and the dynamic agent.   
\begin{figure*}
\centering 
\includegraphics[width=1\textwidth]{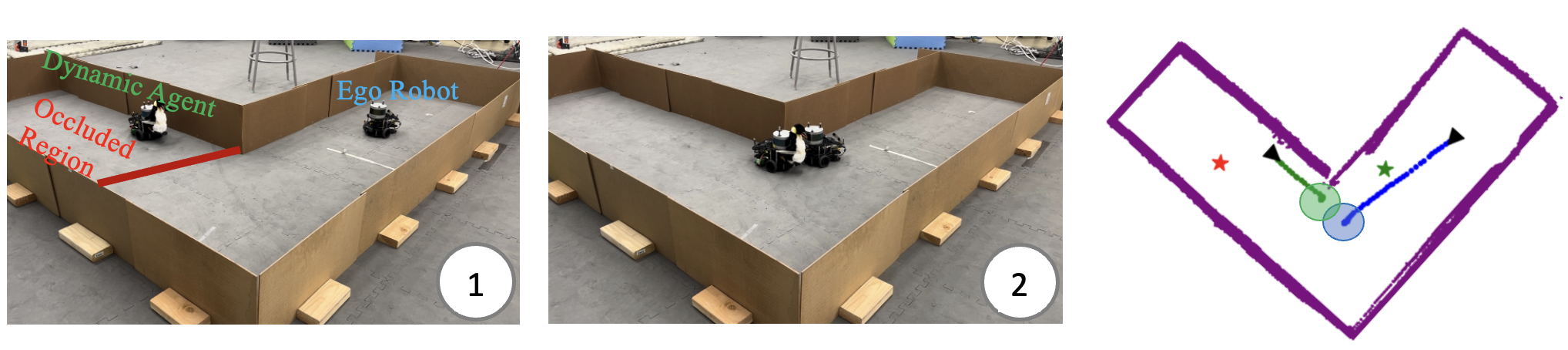}
	\caption{{\textbf{Left: }The ego robot initially is located on the right side, and is unaware of the existence of the dynamic agent shown on the right side, since the corner causes occlusion. \textbf{Middle: }The baseline MPC which is agnostic towards occlusion regions navigates the corner without considering collision avoidance with occluded dynamic agents and collides with the dynamic agent. \textbf{Right: }Blue dots represent the ego robot's executed plan starting from the small black triangle and reaching the goal position (red star). Green dots represent the ego robot's executed plan starting from the small black triangle and reaching the goal position (green star).} }
	\label{fig:EXP_dynamic_agent_baseline}
\end{figure*}

\begin{figure*}
\centering 
\includegraphics[width=1\textwidth]{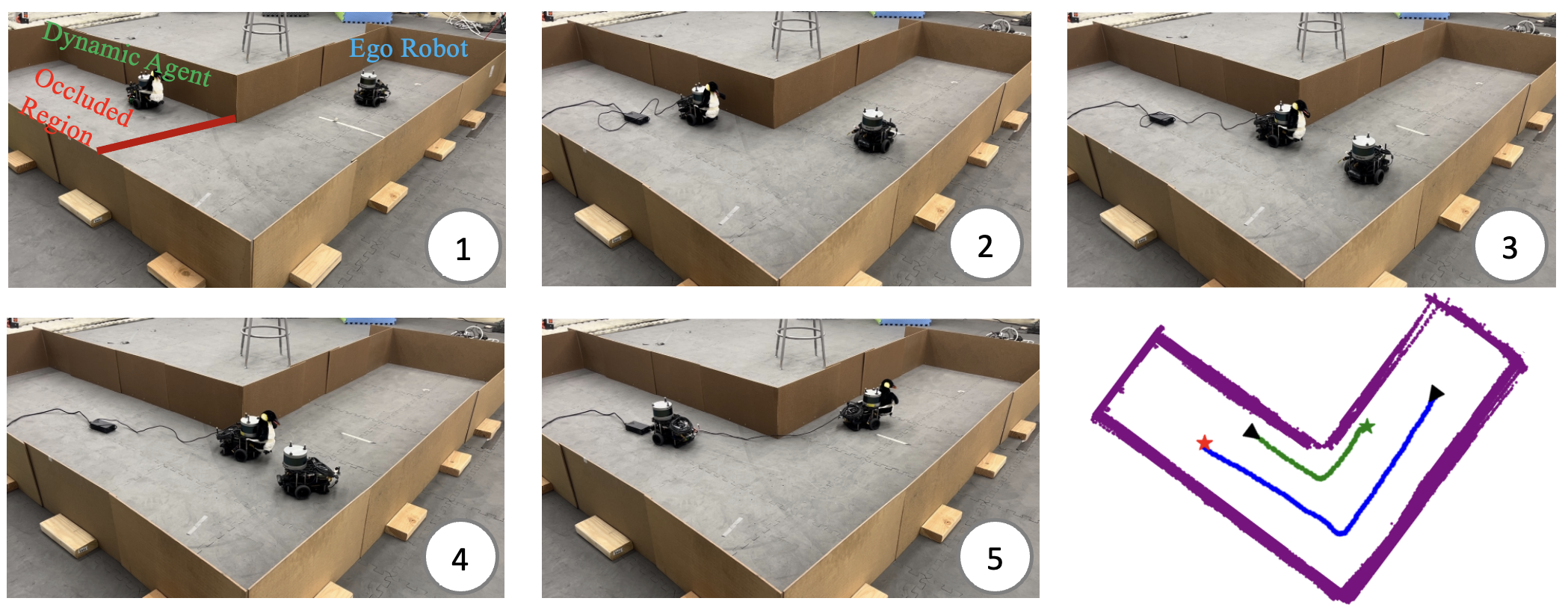}
	\caption{Images labeled as 1,2,3,4, and 5 are sequential snapshots in time. The bottom right image shows the executed trajectories. The ego robot's closed-loop trajectory is displayed in blue, while the dynamic agent's executed trajectory is shown in green. The goal positions for the ego robot and the dynamic agent are represented by red and green stars, respectively, and their initial positions are marked by black triangles. Lidar point clouds are shown in purple.}
	\label{fig:EXP_dynamic_agent_OA_MPC}
\end{figure*}
\begin{figure*}
    \centering 
    \includegraphics[width=1\textwidth]{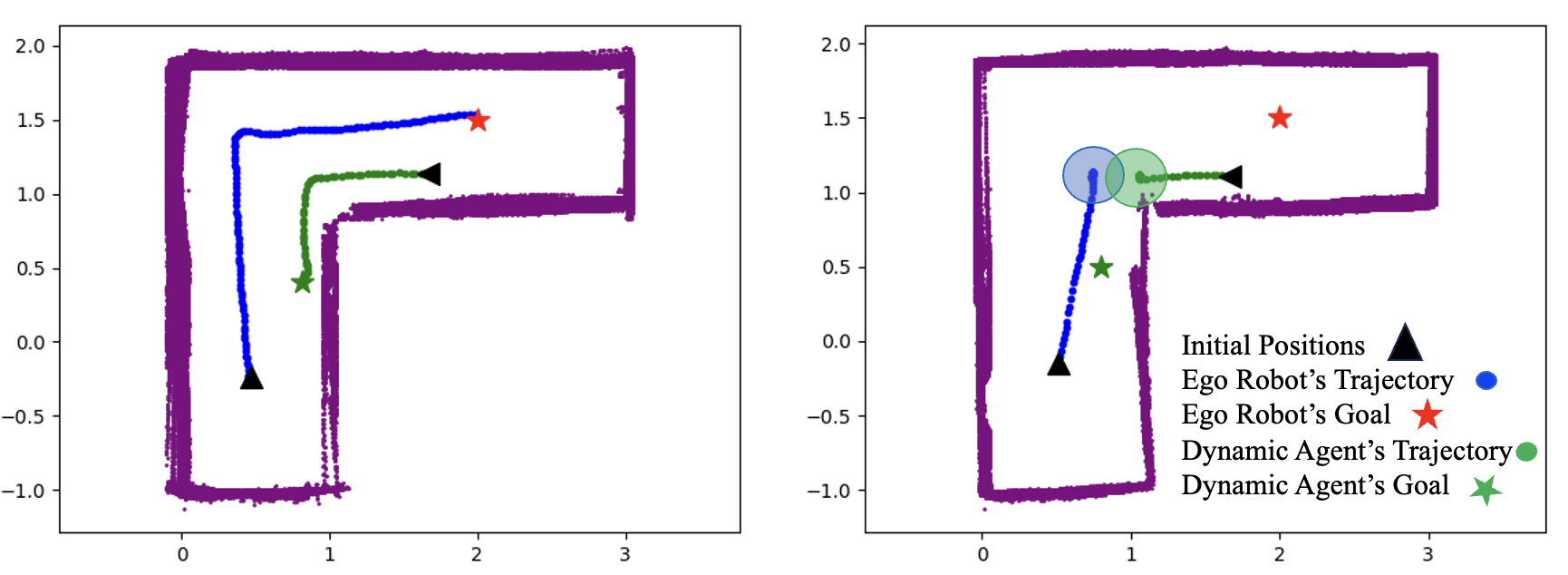}
    \caption{\textbf{Left: }OA-MPC, \textbf{Right: } Baseline MPC. By comparing the results from the OA-MPC and the baseline MPC, as seen the OA-MPC is aware of occluded regions and takes wider turns close to the corner (occluded area). So both ego robot, as well as dynamic agent, can reach their goal positions safely. However, baseline MPC is agnostic to occlusions and takes a tight turn close to the corner which causes a collision with the dynamic agent that is invisible to the ego robot. The Blue and green circles in the right figure represent the ego robot and the dynamic agent locations at the collision point. }
    \label{fig:EXP_dynamic_agent_comparison}
\end{figure*}

% \begin{figure}
% \centering 
% \includegraphics[width=0.9\linewidth]{Pics/complex_scenario_2.png}
% 	\caption{} 
% 	\label{fig:ped_sim_new1}
% \end{figure}

% \begin{figure}
% \centering 
% \includegraphics[width=0.9\linewidth]{Pics/complex_scenario_3.png}
% 	\caption{} 
% 	\label{fig:ped_sim_new1}
% \end{figure}

% \begin{figure}
% \centering 
% \includegraphics[width=0.9\linewidth]{Pics/complex_scenario_4.png}
% 	\caption{} 
% 	\label{fig:ped_sim_new1}
% \end{figure}

\section{Conclusion and Future Work}
\label{SEC8}
This work proposed a novel perception-aware real-time planning framework for the safe navigation of robotics systems in an a priori unknown dynamic environment where occlusions exist. The presented NMPC strategies provide safety guarantees and computational efficiency. For future work, the proposed algorithm can be extended to be used for safe navigation in 3-dimensional spaces for aerial robot applications.

%%%%%%%%%%%%%%%%%%%%%%%%%%%%%%%%%%%%%%%%%%%%%%%%%%%%%%%%%%%%%%%%%%%%%%%%%%%%%%%%
% \section*{APPENDIX}

% Appendixes should appear before the acknowledgment.

% \section*{ACKNOWLEDGMENT}
%%%%%%%%%%%%%%%%%%%%%%%%%%%%%%%%%%%%%%%%%%%%%%%%%%%%%%%%%%%%%%%%%%%%%%%%%%%%%%%%

% \begin{thebibliography}{99}
\bibliographystyle{unsrt}
\bibliography{references}

% \end{thebibliography}

\end{document}